\documentclass{article}

\usepackage{microtype}
\usepackage{graphicx}
\usepackage{subfigure}
\usepackage{booktabs} 

\usepackage{amsmath, amssymb, amsfonts, amsthm}

\DeclareMathOperator*{\argmin}{arg\,min}
\usepackage{natbib}

\usepackage{graphicx}

\usepackage{todonotes}

\usepackage{bbm}
\usepackage{mathrsfs}
\usepackage{stackrel}

\usepackage{hyperref}

\usepackage[capitalise]{cleveref}

\newcommand{\kword}[1]{\noindent\textbf{#1}}

\newcommand{\bbR}{\mathbb{R}} 
\newcommand{\bbE}{\mathbb{E}} 
\newcommand{\variance}{\text{\rm{Var}}}
\newcommand{\Var}{\variance}
\newcommand{\calX}{\mathcal{X}} 
\newcommand{\calY}{\mathcal{Y}} 
\newcommand{\symm}{S} 
\newcommand{\T}{\mathcal{T}}  
\newcommand{\grp}{\mathcal{G}} 
\newcommand{\calE}{\mathcal{E}}
\newcommand{\scD}{\mathscr{D}}
\newcommand{\calL}{\mathcal{L}}

\newcommand{\calD}{\mathcal{D}}
\newcommand{\dgd}{P_{\calD}} 

\newcommand{\haar}{\lambda}
\newcommand{\trdata}{\mathcal{D}^n}
\newcommand{\fclass}{F} 
\newcommand{\invf}[1]{#1^{\circ}} 
\newcommand{\equivf}[1]{#1^{\rm e}}

\newcommand{\borel}{\mathcal{B}} 
\newcommand{\indicator}{\mathds{\indicator}}
\newcommand{\Orbit}{\Phi}

\newcommand{\argdot}{{\,\vcenter{\hbox{\tiny$\bullet$}}\,}} 

\newcommand{\loss}{\ell} 
\newcommand{\risk}{R_{\loss}} 
\newcommand{\eRisk}{\widehat{R}_{\loss}} 
\newcommand{\eRiskAug}{\invf{\widehat{R}}_{\loss}} 

\newcommand{\VarRel}[1]{\stackrel[#1]{}{\Var}}

\newcommand{\eRiskAugMC}{\widehat{R}_{\loss}^{\widehat{\circ}}} 
\newcommand{\KL}[2]{\text{\rm{KL}}(#1 \; || \; #2)} 
\newcommand{\invfMC}[1]{#1^{\widehat{\circ}}}

\def\equdist{\stackrel{\text{\rm\tiny d}}{=}}

\newcommand{\ginv}{$\grp$-invariant}

\newtheorem{theorem}[]{Theorem}
\newtheorem{proposition}[theorem]{Proposition}

\newtheorem{lemma}[theorem]{Lemma}

\usepackage[accepted]{icml2020-arxiv}

\icmltitlerunning{On the Benefits of Invariance in Neural Networks}

\begin{document}

\twocolumn[
\icmltitle{On the Benefits of Invariance in Neural Networks} 



\icmlsetsymbol{equal}{*}

\begin{icmlauthorlist}
\icmlauthor{Clare Lyle}{oxcs}
\icmlauthor{Mark van der Wilk}{mvdw}
\icmlauthor{Marta Kwiatkowska}{oxcs}
\icmlauthor{Yarin Gal}{oxcs}
\icmlauthor{Benjamin Bloem-Reddy}{ubc}
\end{icmlauthorlist}

\icmlaffiliation{oxcs}{Department of Computer Science, University of Oxford, Oxford, United Kingdom}
\icmlaffiliation{mvdw}{Department of Computing, Imperial College London, London, United Kingdom}
\icmlaffiliation{ubc}{Department of Statistics, University of British Columbia, Vancouver, Canada}

\icmlcorrespondingauthor{Clare Lyle}{clare.lyle@univ.ox.ac.uk}


\vskip 0.3in
]



\printAffiliationsAndNotice{}  

\begin{abstract}
Many real world data analysis problems exhibit invariant structure, and models that take advantage of this structure have shown impressive empirical performance, particularly in deep learning. While the literature contains a variety of methods to incorporate invariance into models, theoretical understanding is poor and there is no way to assess when one method should be preferred over another. In this work, we analyze the benefits and limitations of two widely used approaches in deep learning in the presence of invariance: data augmentation and feature averaging. We prove that training with data augmentation leads to better estimates of risk and gradients thereof, and we provide a PAC-Bayes generalization bound for models trained with data augmentation. We also show that compared to data augmentation, feature averaging reduces generalization error when used with convex losses, and tightens PAC-Bayes bounds. We provide empirical support of these theoretical results, including a demonstration of why generalization may not improve by training with data augmentation: the `learned invariance' fails outside of the training distribution. 
\end{abstract}

\section{Introduction}

Many real-world problems exhibit invariant structure. Tasks involving set-valued inputs such as point clouds are invariant to permutation. Image classification tasks are often rotation- and translation-invariant. Intuitively, models that capture the invariance of a problem should perform better than those that do not. This is supported by empirical results in a range of applications \cite{Cohen:Welling:2016, fawzi2016adaptive, salamon2017deep}.  

There are many ways of incorporating invariance into a model. One can build the invariance into the network as a convolution or weight-tying scheme, or average network predictions over transformations of the input (feature averaging), or simply train on a dataset augmented with these transformations (data augmentation). Each of these approaches has been demonstrated to perform well in various settings, but there remains a large divide between their impressive practical performance and solid theoretical understanding. 

The lack of theory leaves open a number of questions. Firstly, if invariance is incorporated into a model or training algorithm, what are the theoretical guarantees on the performance of the trained model? Relatedly, as a matter of practice, how should a practitioner choose amongst the different approaches to incorporating invariance? 
Concretely, if an invariant model and a model trained with data augmentation both attain the same training error, which one should be preferred? Can one or the other be expected to converge faster? 
These questions are the key motivation for our work. 

We focus the two most generically applicable methods, data augmentation  and feature averaging. 
Our overall conclusion is that \textbf{feature averaging is better than data augmentation is better than doing nothing}; this holds even for stochastic (Monte Carlo) approximations of the averages involved in feature averaging and data augmentation. On the journey to the main conclusion, we uncover a number of intriguing properties and shed light on the mathematical structure driving the impressive practical performance of the methods.

\begin{table*}[bt]
  \vspace{-0.15in}
  \caption{Summary of theoretical results.}
  \label{tab:theory:summary}
  \begin{center}
  \resizebox{\textwidth}{!}{
  	\begin{tabular}{lccccc}
  	  \toprule
  	    & \textbf{Baseline} &  & \textbf{Data Augmentation} &  & \textbf{Feature Averaging} \\
  	  \midrule
  	  Expected risk & $\risk(f)$ & $=$ & $\risk(f)$ & $\overset{\textrm{convex $\loss$}}{\geq}$ & $\risk(\invf{f})$ \\
  	  & & \small{\cref{prop:da:variance:reduction}} & & \small{\cref{prop:empirical:risk:order}} &  \\
  	  \midrule
  	  Empirical risk & $\eRisk(f,\trdata)$ & & $\eRiskAug(f,\trdata)$ & $\overset{\textrm{convex $\loss$}}{\geq}$ & $\eRiskAug(\invf{f},\trdata) = \eRisk(\invf{f},\trdata)$  \\
  	  & & & & \small{\cref{prop:empirical:risk:order}} &  \\
  	  \midrule
  	  Variance of $\eRisk$  & $\VarRel{\trdata\sim\dgd}[\eRisk(f,\trdata)]$ & $\geq$ & $\VarRel{\trdata\sim\dgd}[\eRiskAug(f,\trdata)]$ & $\overset{\textrm{convex $\loss$}}{\geq}$ & $\VarRel{\trdata\sim\dgd}[\eRisk(\invf{f},\trdata)]$ \\
  	   & & \small{\cref{prop:da:variance:reduction}} & & \small{\cref{prop:empirical:risk:order}} &  \\
  	  \midrule
  	  KL term in & $\KL{Q}{P}$ & $=$ & $\KL{Q}{P}$ & $\geq$ & $\KL{\invf{Q}}{\invf{P}}$ \\
  	  PAC-Bayes bound & & & & \small{\cref{lemma:KL:gen}} & \\
  	  \midrule
  	  PAC-Bayes bound & $B_0$ & $=$ & $B_{\textrm{DA}}$ & $\geq$ & $B_{\textrm{FA}}$ \\
  	  for 0-1 loss & & \small{\cref{thm:pac:bayes:da}} & & \small{\cref{thm:bound:order}} & \\
  	  \midrule
  	  Monte Carlo approx. & & & PAC-Bayes bound holds & & $\KL{Q}{P} \geq \KL{\invfMC{Q}_{G^k}}{\invfMC{P}_{G^k}}$ \\
  	  ($k\geq 1$ samples) & & & \small{\cref{thm:pac:bayes:da}} & & $\geq \KL{\invf{Q}}{\invf{P}}$ \\
  	  \bottomrule
  	\end{tabular}
  	}
  \end{center}
  \vspace{-.2in}
\end{table*}

\subsection{Summary of Results}

We consider the data-generating distribution $\dgd$ to be invariant to the action of a group $\grp$: $\dgd(gX, Y) = \dgd(X,Y)$, for all $g \in \grp$ (see \cref{sec:background} for details). Our main results relate baseline training of a generic neural network (or other predictive model) via empirical risk minimization (ERM) to performing either data augmentation or feature averaging. \Cref{tab:theory:summary} summarizes the theoretical results.

\textbf{Data augmentation} (DA) (\cref{sec:data:augmentation}) improves on baseline training with ERM by minimizing an {augmented risk}, the risk averaged over the orbits induced by $\grp$. This yields a lower-variance {estimator} of the model risk and its minima (\cref{prop:da:variance:reduction}). The variance reduction also applies to gradients of the risk, and therefore affects gradient-based learning. Our results to this end are essentially the same as some by \citet{chen2019invariance}. In contrast to that work, we investigate PAC-Bayes bounds for generalization of DA. Traditional PAC-Bayes bounds based on i.i.d.\ data do not apply to DA because the augmented dataset violates the i.i.d.\ assumption. We show that the i.i.d.\ bounds also apply to DA and in particular to the augmented risk (\cref{thm:pac:bayes:da}), and that tighter bounds may be possible. However, training with DA is not guaranteed to produce an invariant (or even approximately invariant) function. We demonstrate empirically how this can fail; we also provide an example where minimizing the augmented risk yields an invariant function (\cref{sec:da:ood}).

\textbf{Feature averaging} (FA) (\cref{sec:feature:averaging}) yields a lower-entropy {function class}. In the case of convex loss, FA also obtains lower expected risk than DA and lower-variance estimates of risk and its gradient (\cref{prop:empirical:risk:order}). Furthermore, symmetrization compresses the model, and thus tightens PAC-Bayes bounds by reducing the KL term, a phenomenon that holds even for Monte Carlo approximations to FA (\cref{lemma:KL:gen,prop:kl:chain}). As a byproduct, we prove a general lemma (\cref{lem:pushforward:KL}) that connects the present paper to work on generalization and post-training compression \citep{zhou2018nonvacuous}. 

We illustrate our theoretical results with experiments in \cref{sec:experiments}, and also investigate practical questions raised by the theory. We conclude (\cref{sec:conclusion}) by interpreting the theory and experiments as practical recommendations.

\section{Background} \label{sec:background}

``Invariance'' has been used to describe a number of related but distinct phenomena in the machine learning and statistics literature. One perspective, which is shared by the present work, considers invariance of a neural network's output with respect to a group acting on its inputs \citep[e.g.,][]{Cohen:Welling:2016,Kondor:Trivedi:2018,invariantdistributions}.\footnote{These ideas (and the results in the present work) apply generally to functions, and therefore to a broader set of machine learning techniques; we focus on neural networks for continuity with the previous literature.} Other work has used looser notions. For example, \citet{zou2012deep} use ``invariant'' to mean ``not changing very much''. Related ideas are ``local invariance'' \citep{raj2017orbit_embeddings}, ``insensitivity'' \citep{invariantgps}, and ``approximate invariance'' \citep[][Sec.~6]{chen2019invariance}. 

We focus on invariance under the action of a group $\grp$. The {action} of $\grp$ on a set $\calX$ is a mapping $\alpha : \grp \times \calX \to \calX$ which is compatible with the group operation. For convenience, we write $\alpha(g,x) = \alpha_x(g) = gx$, for $g \in \grp$ and $x \in \calX$. The \kword{orbit} of any $x\in\calX$ is the subset $\grp_x$ of $\calX$ that can be obtained by applying an element of $\grp$ to $x$, $\grp_x = \{ gx : g \in \grp \}$. For mathematical simplicity, we assume $\grp$ to be compact, with
(unique) normalized Haar measure denoted by $\haar$.\footnote{$\haar$ is analogous to the uniform distribution on $\grp$. Our results generalize---with some additional technicalities---to any group that acts property on $\calX$ and has Haar measure.} 
We denote a random element of $\grp$ by $G$. 
A mapping $f : \calX \to \calY$ is \kword{invariant} under $\grp$ (or $\grp$-invariant) if
\begin{align} \label{eq:invariant:function}
  f(gx) = f(x) \;, \quad g\in \grp,\ x \in \calX \;.
\end{align}
Any function $f: \calX \to \bbR$ can be \kword{symmetrized} by averaging over $\grp$. We denote this with a symmetrization operator $\symm_{\grp}$, defined as
\begin{align} \label{eq:symmetrization:def}
  \invf{f}(x) := \symm_{\grp}f(x) = \bbE_{G\sim\haar}[f(G x)] \;, \quad x \in \calX \;.
\end{align}

We consider a typical machine learning scenario, with a training data set $\trdata$ of $n$ observations $(X_i, Y_i)_{i=1}^n\in (\calX,\calY)^n$ sampled i.i.d.\ from some (unknown) probability distribution $\dgd$. 
Furthermore, $\dgd$ \emph{is known or assumed to be $\grp$-invariant},
\begin{align} \label{eq:invariant:dgd}
  \dgd(gX, Y) = \dgd(X ,Y ) \;, \quad g \in \grp \;.
\end{align}
For example, $X$ may be an image of an animal, $Y$ a label of the animal, and $\grp$ the group of two-dimensional rotations. 

The marginal distribution on $X$ of any \ginv\ $\dgd$ has a disintegration into a distribution $P_{\Orbit}$ over orbits of $\calX$, each endowed with an \kword{orbit representative} $\Orbit\in\calX$, and a conditional distribution $P_{X | \Orbit}=\haar \circ \alpha^{-1}_{\Orbit}(\argdot)$ induced by applying a random $G \sim \haar$ to $\Orbit$ \citep[see, e.g.,][]{invariantdistributions}. That is, $(X,Y) \equdist (G\Orbit,Y)$ and $\dgd = P_{\Orbit} \times P_{X | \Orbit} \times P_{Y|X}$. 
The specific relevance to this work is that expectations with respect to $\dgd$ can be iterated as
  \begin{align} \label{eq:iterated:expectations}
  	& \bbE_{(X,Y)\sim\dgd}[f(X,Y)] \\
  	&\ \ = \bbE_{Y\sim P_{Y|X}}[\bbE_{\Orbit\sim P_{\Orbit}}[ \bbE_{G\sim\haar}[f(G\Orbit,Y) \mid \Orbit,Y] \mid Y ] ] \;. \nonumber
  \end{align}

For a class of functions $F= \{f : \calX \to \calY\}$, a probability distribution $Q$ on $F$, and a loss function $\loss : \calY \times \calY \to \bbR_+$. We denote various expected and empirical risks as follows:
\begin{align*}
  \risk(f) &= \mathbb{E}_{(X,Y)\sim\dgd}[\ell (f(X), Y)] \\
   \risk(Q) &= \mathbb{E}_{f \sim Q}[ \risk(f)] \\
  \eRisk(f, \trdata) &= \textstyle\frac{1}{n}\textstyle\sum_{i=1}^n \ell(f(X_i), Y_i) \\
   \eRisk(Q, \trdata) &= \mathbb{E}_{f\sim Q} [\eRisk(f, \trdata)]
\end{align*}

\subsection{Modes of Invariance}

Common sense indicates that when modeling \ginv\ $\dgd$, any good model will also be \ginv, at least to a good approximation. This has been achieved in practice through one of three approaches: trained invariance, encouraged during training via DA; network symmetrization, typically implemented as FA; and symmetric network design, obtained by composing a \ginv\ layer with a sequence of $\grp$-equivariant layers. 

\kword{Trained invariance} is implemented as DA \citep{fawzi2016adaptive,Cubuk2018}: (possibly random) elements $G_{ij}$ of $\grp$ are applied to each observation $X_i$ of the training data, with the label $Y_i$ left unchanged. The result is an augmented dataset $\trdata_{\grp} = ( (G_{ij}  X_i, Y_i)_{j \leq m} )_{i\leq n}$ used to minimize the augmented empirical risk
\begin{align} \label{eq:aug:risk}
  \eRiskAug(f,\trdata) & = \frac{1}{n}\sum_{i=1}^n  \bbE_{G\sim\haar}[\loss(f(G  X_i),Y_i) ] \\
  & \approx \frac{1}{nm} \sum_{i=1}^n \sum_{j=1}^m \loss(f(G_{ij} X_i),Y_i) \nonumber \;.
\end{align}
DA is now a standard method in practitioners' toolkit \citep{iyyer-etal-2014-neural,zhou2015predicting,salamon2017deep, machinehealth}, particularly due to its ease-of-implementation and flexibility: $\grp$ may be a set of transformations that is not a group, which permits its use for encouraging exact or approximate invariance under an arbitrary set of transformations. 
Networks trained with DA have been observed to exhibit greater invariance to the desired transformations than those trained on the original dataset \citep{fawzi2016adaptive} despite the fact that invariance is not part of the built-in network architecture.
Moreover, it can have positive effects on generalization even when the augmentation transformations are not present in the test set \citep{zhang2016understanding}.

Theoretical understanding of DA is still being developed. 
Recent theoretical work has established connections to FA and variance reduction methods. Specifically, \citet{kerneltheory} showed that for a kernel linear classifier, minimizing the augmented risk is equivalent, to first order, to minimizing the feature averaged risk; and that a second-order approximation to the objective is equivalent to data-dependent variance regularization. \Citet{chen2019invariance} showed that averaging over the set of transformations is a form of Rao--Blackwellization, and the resulting variance-reduction yields a number of desirable theoretical statistical properties. 

\kword{Architectural invariance} restricts the function class being learned to contain only invariant functions, typically through either  FA or symmetric function composition. FA relies on computing an average over $\grp$ at one or more layers, such that the overall network is invariant under $\grp$ acting on the input. 
In practice, averaging is typically done at the penultimate or final layer, resulting in a $\grp$-invariant network ${f}^{\circ}$. A network $f$ with $D$ layers is written as the composition of $h_D \circ \dots \circ h_1$, with the shorthand $h_{d}^{d'}$ referring to the composition of layers $d$ through $d'$. The empirical risk of a network with FA at layer $d$ evaluated on $\trdata$ is
\begin{align*} 
  \eRisk(\invf{f},\trdata) 
    = \frac{1}{n}\sum_{i=1}^n \loss\big(h_{d}^D\circ  \bbE_{G\sim\haar}[ h_{1}^{d-1} (G X_i) ], Y_i \big) \;.
\end{align*}

As with DA, FA can be applied to approximate and non-group invariance. The average over $\grp$ might also be estimated by applying randomly sampled elements of $\grp$, though when $h^D_{d}$ is nonlinear the estimate of $\invf{f}$ may be biased. Unlike DA, symmetrization guarantees that the output function $f^\circ$ will be invariant to $\grp$ whenever the expectation over $\grp$ can be computed exactly. The exact computation of this expectation, however, can be computationally expensive (linear in $|\grp|$ when discrete) or even intractable (when $\grp$ is infinite), in which case Monte Carlo estimates can be used. 

The elegant, albeit less generically applicable approach of symmetric function composition uses properties of $\grp$ to determine particular functional forms that are {equivariant} or invariant under $\grp$. An invariant network $\invf{f}$ is constructed by composing an invariant function (layer) $\invf{h}$ with a sequence of equivariant functions $(\equivf{h}_k)$: $\invf{f} = \invf{h} \circ \equivf{h}_D \circ \dotsb \circ \equivf{h}_1$. A body of literature of varying degrees of generality has developed \citep{Wood:ShaweTaylor:1996,ravanbakhsh2017equivariance,Kondor:Trivedi:2018,invariantdistributions,cohenetal2019generaltheory}. This includes convolutional networks. 
Empirical results indicate that this approach has advantages over trained invariance \citep[e.g.,][]{Cohen:Welling:2016}. Theoretical results to this end are lacking, with the notable exception of the VC-dimension-based PAC bounds obtained by \citet{ShaweTaylor:1991:PAC,ShaweTaylor:1995:SampleSizes}, which connect a tighter generalization bound to the reduction in parameters that results from symmetry constraints. We do not consider equivariant-invariant architectures further, and leave their theory as future work.

\subsection{PAC-Bayes Generalizations Bounds}
\label{sec:pac:bayes}
Understanding the generalization performance of deep learning models is a core research objective of modern machine learning. Many empirical results appear counterintuitive, and remain largely unexplained by theory. Networks with many more parameters than observations may generalize well, despite also having the capacity to memorize the training set \citep{zhang2016understanding}. Uniform generalization bounds often result in \emph{vacuous} bounds, i.e., they are greater than the upper bound of the loss function \citep{dziugaite2017nonvacuous}. 
However, PAC-Bayes bounds \citep{McAllester_1999} have been successfully applied to large deep network architectures to obtain nonvacuous generalization guarantees \citep{dziugaite2017nonvacuous,dziugaite2018dependent,zhou2018nonvacuous}.

PAC-Bayes bounds characterize the risk of a randomized prediction rule; the randomizaton is interpreted as a Bayesian posterior distribution $Q$ that can depend on $\trdata$. The typical bound on generalization error is expressed in terms of the empirical risk and the KL divergence between $Q$ a fixed prior distribution $P$. The following is a standard bound due to \citet{catoni}, which holds for general data generating distributions $\dgd$ and 0-1 loss.

\begin{theorem}[\citet{catoni}]\label{thm:catoni:bound}
  Let $\mathcal{D}^n$ be sampled i.i.d. from $\dgd$, and let $\loss$ be 0-1 loss. 
  For any prior $P$ and any $\delta\in (0,1)$, with probability $1-\delta$ over samples $\trdata$,  for all posteriors $Q$ and for all $\beta > 0$,
	\begin{equation} \label{eq:catoni:bound}
	     \risk(Q) \leq \frac{ 1 - e^{-\beta \eRisk(Q, \trdata) - \frac{1}{n}(\KL{Q}{P} + \log \frac{1}{\delta})} }{1 - e^{-\beta}} \;.
	\end{equation}
\end{theorem}

For bounded loss functions, analogous bounds are in terms of the so-called KL generalization error \citep[see, e.g.,][]{dziugaite2017nonvacuous,dziugaite2018dependent}. We state all results only for variations of Catoni's bound \eqref{eq:catoni:bound}, but versions for KL generalization error are straightforward to derive.

\section{Data Augmentation Reduces Variance} \label{sec:data:augmentation}

In this section, we discuss the ways in which DA performs better than baseline ERM, and establish the validity of a PAC-Bayes bound for models trained with DA. 

Recently, \citet{chen2019invariance} established that when $\dgd$ is \ginv, DA reduces the variance of ERM-based estimators by (approximately) averaging loss over the orbits of the observations, which can be seen as a form of Rao--Blackwellization. Specifically, for any integrable $h : \calX \times \calY \to \bbR$, symmetrizing is equivalent to taking the conditional expectation, conditioned on the orbit of $X$:
\begin{align*}
  \bbE_{G\sim\haar}[h(GX,Y)] = \bbE_{(X,Y)\sim\dgd}[h(X,Y) \mid \Orbit(X)] \;.
\end{align*}
The average of $\grp$ appears in the augmented empirical risk \eqref{eq:aug:risk}, and reduces the variance of risk estimates. Specifically, the variance of the risk decomposes into within-orbit and across-orbit terms, and the within-orbit term vanishes for the augmented risk. The result follows directly from \citet[][Lemma 4.1]{chen2019invariance}; we also give a proof in \cref{appx:proofs} that highlights the structure of the problem. 
\begin{proposition}[\citet{chen2019invariance}] \label{prop:da:variance:reduction}
  If $\dgd$ is \ginv\ and $\loss(f(\argdot),\argdot) \in L_2(\dgd)$, then 
  \begin{align*} 
  \begin{gathered}
  	\bbE_{\trdata\sim\dgd^n}\big[ \eRiskAug(f,\trdata)  \big] = \bbE_{\trdata\sim\dgd^n}\big[ \eRisk(f,\trdata)  \big] \;, \text{ and} \\ 
  	\Var_{\trdata\sim\dgd^n}\big[ \eRiskAug(f,\trdata)  \big] \leq \Var_{\trdata\sim\dgd^n}\big[ \eRisk(f,\trdata)  \big] \;. 
  \end{gathered}
  \end{align*}
\end{proposition}

\subsection{Practical Data Augmentation}

Computing $\bbE_{G\sim\haar}[\loss(f(GX),Y)]$ exactly may be infeasible: $\grp$ may be discrete but large, or $\grp$ may be continuous. In either case, practical DA relies on Monte Carlo estimates, typically within stochastic gradient descent (SGD). Specifically, with $G_{ij} \sim \haar$, $\eRiskAug$ is approximated by
\begin{align} \label{eq:aug:risk:mc}
  \eRiskAugMC(f,\trdata) := \frac{1}{nm} \sum_{i=1}^n \sum_{j=1}^m \loss(f(G_{ij} X_i),Y_i) \;.
\end{align}
For ``nice'' loss functions---those for which we can interchange differentiation and $\bbE_{G\sim\haar}$---the symmetrization reduces the variance of gradient estimates of augmented risk. Conversely, the variance of the Monte Carlo estimate of \eqref{eq:aug:risk:mc} may offset the reduction obtained from averaging. Furthermore, it has been argued that the noise in SGD implicitly regularizes the objective \citep{neyshabur2017thesis}; excessive variance reduction may have harmful effects. In short, the consequences of the interplay between the variance reduction of symmetrization and the variance increase of approximating that symmetrization, especially in the context of SGD for overparameterized models, are not clear. The details of those trade-offs are beyond the scope of this paper; we briefly investigate the effects empirically in \cref{sec:experiments}.

\subsection{Data Augmentation and Trained Invariance} \label{sec:da:ood}

While DA is sometimes referred to as an approach to train an invariant function, the learned function will not be invariant in general. The objective of training with DA is to minimize a symmetrized risk, not to find a symmetric function. 

One setting in which minimizing the augmented risk will yield an invariant function is with a linear model $f_w(X) = w^{\top}X$ and a convex loss, and with $\grp$ a group whose action on $\calX$ has a linear representation. To state the result, let $V$ be a $d$-dimensional vector space over $\bbR$ with dual vector space $V^*$, and assume that $\calX$ spans $V$. Furthermore, let $\grp$ admit a linear representation, $\rho : \grp \to GL(V)$, with corresponding dual $\rho^*_g = \rho_{g^{-1}}^{\top}$. 

\begin{proposition}\label{theorem:symmgd}
  Suppose that $\grp$ has a linear representation, as described above, let $f_w(X) = w^{\top}X$ and $\loss$ be strictly convex. Then the (global) minimizer $\hat{w}$ satisfies $\rho_g^* \hat{w} = \hat{w}$ for $\haar$-almost all $g\in\grp$. In particular, $f_{\hat{w}}$ is $\grp$-invariant.
\end{proposition}
Under suitable step-size conditions (e.g., the Robbins--Munro conditions) SGD will converge to an invariant set of weights. Thus, to learn a predictor that exhibits the desired invariance on the entire dataset, it is sufficient to train with SGD on augmented data with a convex loss.

Non-convex objectives with non-linear models do not yield similar results. As most settings for which we would use deep learning are both non-convex and non-linear, this suggests that although DA may appear to promote invariance, it may fail to learn networks that are truly invariant. 

The example depicted in \cref{fig:ood} demonstrates such a failure: the learned function appears to capture the target invariance on the training data, but, having not learned the appropriate symmetry in weight space, fails to generalize to novel data and displays high variance over orbits in evaluation. We train fully connected neural networks using DA on one of two related datasets: MNIST and fashionMNIST ($28 \times 28$ pixel black and white images of handwritten digits and clothing categories respectively), each augmented by rotations of multiples of 90 degrees. We then evaluate the variance of the outputs over orbits (rotations by 90, 180, and 270 degrees) in the test set. Finally, we evaluate the two networks on orbits in the complementary dataset.

We observe that the networks attain low variance over orbits of data drawn from the same distribution as the training data. The performance of the networks on out-of-distribution data is more interesting. The MNIST network has increasingly \textit{higher} variance of its predictions on the rotations of fashionMNIST as it reduces its prediction variance over orbits of MNIST. We also note that the variance between random seeds in the variances over orbits was significantly higher on the out-of-distribution data. We omit data for FA because averaging over each of the four rotations of the input trivially yields a variance of zero over each orbit.
\begin{figure}[t]
    \centering
    \vspace{-0.123in}
    \includegraphics[width=0.48\textwidth]{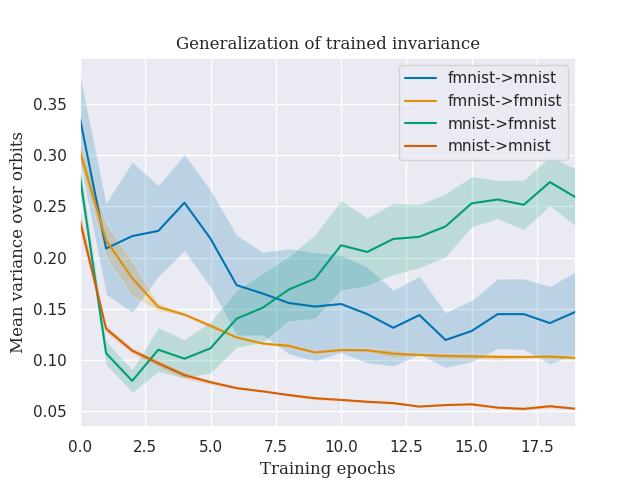}
    \vspace{-1.5\baselineskip}
    \caption{Variance of predictions w.r.t. rotations of input over the course of training. Labels indicate training $\rightarrow$ evaluation set.}
    \label{fig:ood}
    \vspace{-0.175in}
\end{figure}

\subsection{PAC-Bayes Generalization Bounds}

The PAC-Bayes bound in \cref{thm:catoni:bound} holds with binary classification loss. Exact DA violates the assumptions because with the same loss, $\bbE_{G\sim\haar}[\loss(f(GX),Y)] \in [0,1]$. Monte Carlo approximations of $\bbE_{G\sim\haar}[\loss(f(GX),Y)]$ also violate the assumptions of \cref{thm:catoni:bound} because the augmented data set is not i.i.d. 
We address both issues with the following {PAC-Bayes bound for DA}.

\begin{theorem}
  \label{thm:pac:bayes:da}
  Assume that $\dgd$ is \ginv. Then \cref{thm:catoni:bound} holds with either of $\eRiskAug(Q,\trdata)$ as in \eqref{eq:aug:risk} or $\eRiskAugMC(Q,\trdata)$ as in \eqref{eq:aug:risk:mc} substituted for $\eRisk(Q,\trdata)$.
\end{theorem}
See \cref{appx:proofs} for the proof, which uses a general formula of \citet{leveretal2013tighterPACbayes} and the invariance structure of $\dgd$. 
The bound \eqref{eq:catoni:bound} is looser than what is theoretically possible for DA.  However, tighter bounds with an analytic form (see \cref{appx:tighter:pacbayes:da}) are computationally intractable. 

\section{Feature Averaging Can Do More} \label{sec:feature:averaging}

In this section we establish that FA should be preferred over DA in most situations. When the loss is convex, generalization error decreases both in expectation and per-dataset, and there is a further variance-reduction in risk estimates. More importantly, symmetrization compresses the model, resulting in a \kword{symmetrization gap} in the PAC-Bayes bound. 

For a group $\grp$ that acts on $\calX$, symmetrization of any function $f : \calX \to \bbR$ can be performed by averaging over $\grp$, as in \eqref{eq:symmetrization:def}. 
Fix a function class $\fclass$, and let $\invf{\fclass}$ denote the class of $\grp$-invariant functions obtained by symmetrizing the functions belonging to $\fclass$. Clearly, $\symm_{\grp}$ is surjective, but it may not be injective. The inverse image of $\invf{f}$, $\symm_{\grp}^{-1} \invf{f}$, yields the set of functions in $\fclass$ whose $\grp$-symmetrization yields $\invf{f}$. Function symmetrization is naturally extended to probability measures on function classes: 
for any probability measure $P$ on $\fclass$, the induced probability measure on $\invf{\fclass}$ is the image of $P$ under $\symm_{\grp}$, $\invf{P} = P \circ \symm_{\grp}^{-1}$.

\subsection{Further Variance Reduction with Convex Loss}

With convex loss, Jensen's inequality can be applied to the augmented risk to compare DA and FA risk estimates. The proof of the following proposition is given in \cref{appx:proof:prop:empirical:risk:order}.

\begin{proposition} \label{prop:empirical:risk:order}
  Let $\loss : \bbR \times \bbR \to \bbR_+$ be a loss function that is convex in its first argument. 
  Then for any $f : \calX \to \calY$, 
  \begin{align*} 
    \eRisk(\invf{f},\trdata) = \eRiskAug(\invf{f},\trdata)  \leq \eRiskAug(f,\trdata) \;,
  \end{align*}
  and therefore analogous inequalities hold for $\eRisk(\invf{Q},\trdata)$, $\risk(f)$, and $\risk(Q)$. 
  Furthermore, if $\loss(f(\argdot),\argdot) \in L_2(\dgd)$ (i.e., has finite second moment), 
  \begin{align*}
  	\Var_{\trdata\sim\dgd^n}\big[ \eRisk(\invf{f},\trdata)  \big] \leq \Var_{\trdata\sim\dgd^n}\big[ \eRiskAug(f,\trdata)  \big] \;.
  \end{align*}
\end{proposition}

\subsection{Reduction in KL via the Symmetrization Gap} 

In modern deep learning architectures, one typically has sufficiently large capacity to drive the empirical risk arbitrarily close to zero. Although the variance-reduction of the previous section can help during training, the dominant term in the generalization bound \eqref{eq:catoni:bound} is $\KL{Q}{P}$. Indeed, much of the recent literature on obtaining nonvacuous PAC-Bayes bounds focuses on minimizing this term, subject to not overly inflating the empirical risk. 

Consider the approach of \citet{zhou2018nonvacuous}: train a deep neural network, and use a compression algorithm to obtain a lossy compression of the trained network. Countering the potential for deterioration in the empirical risk, the KL term applied to the compressed network achieves a massive reduction in entropy; the compressed network is much less complex. Those basic concepts apply more generally, as formalized by the following lemma. Although fundamental, we have been unable to find a published proof (though it would be surprising if one does not exist). We give the proof in \cref{appx:proof:lem:pushforward:KL}.

\begin{lemma} \label{lem:pushforward:KL}
  Suppose that $(E_i,\calE_i)$, $i=1,2$, are two measurable spaces, the second of which is standard, $\mu$ and $\nu$ are two probability measures on $(E_1,\calE_1)$, and $\psi : (E_1,\calE_1) \to (E_2,\calE_2) $ is a measurable map. Then
  \begin{align} \label{eq:kl:inequality:pushforward}
  	\KL{\mu\circ\psi^{-1}}{\nu\circ\psi^{-1}} \leq \KL{\mu}{\nu} \;.
  \end{align}
  Furthermore, if $\mu\ll \nu$ with density $m$, then $\mu\circ\psi^{-1} \ll \nu\circ\psi^{-1}$ with density $m_{\psi}$, and the \kword{$\psi$-gap} is
  \begin{align}
  	\Delta_{\psi}(\mu\ ||\ \nu)  : & = \KL{\mu}{\nu} - \KL{\mu\circ\psi^{-1}}{\nu\circ\psi^{-1}} \nonumber \\
  	  & = \int_{E_1} \mu(dx) \log \frac{m(x)}{(m_{\psi}\circ\psi)(x)} \;.
  \end{align}
\end{lemma}

In particular, when $\psi$ is non-injective, points of $(E_1,\calE_1)$ become equivalent; $(E_2,\calE_2)$ is a compressed version, and the probability measures $\mu$ and $\nu$ are similarly compressed. 

\textbf{The symmetrization gap.} 
Applying \cref{lem:pushforward:KL} with $\psi = \symm_{\grp}$ indicates that symmetrization can reduce the KL divergence term in the PAC-Bayes bound. 

\begin{theorem} \label{lemma:KL:gen}
  Let $\mathcal{X}$ be a compact metric space and $\mathcal{Y}$ a Polish space, $\grp$ a group acting measurably on $\mathcal{X}$, and $\fclass = C(\mathcal{X},\mathcal{Y})$ the class of continuous functions $\mathcal{X}\to \mathcal{Y}$.\footnote{The result can hold for other function classes $\fclass$; the key requirement is that conditioning is properly defined in $\fclass$ and $\invf{\fclass}$.}  
  Let $Q$ and $P$ be probability measures on $\fclass$ such that $Q \ll P$ with density $q$, and $\invf{Q}\ll\invf{P}$ (density $\invf{q}$) their images under $\symm_{\grp}$ on $\invf{\fclass}$. Then 
  \begin{equation*}
    \KL{\invf{Q}}{\invf{P}} \leq \KL{Q}{P} \;.
  \end{equation*}
  Furthermore, the \kword{symmetrization gap} is 
  \begin{align} \label{eq:symm:gap}
  	\invf{\Delta}(Q \ ||\ P) = \bbE_{f\sim Q}\bigg[ \log \frac{q(f)}{\invf{q}(\symm_{\grp} f)}  \bigg] \;.
  \end{align}
\end{theorem}

Because $\invf{Q}$ is the image of $Q$, the densities in \eqref{eq:symm:gap} satisfy
\begin{align} \label{eq:density:matching}
  \int_{\symm_{\grp}^{-1}B} q(f) P(df) = \int_{\symm_{\grp}^{-1}B} \invf{q}(\symm_{\grp}f) P(df) \;,
\end{align}
for all sets $B$ in the $\sigma$-algebra on $\invf{\fclass}$. Although this imposes a large number of constraints on $q$ and $\invf{q}\circ\symm_{\grp}$, they may differ greatly across $\fclass$. In particular, consider a $\grp$-induced equivalence class $\symm_{\grp}^{-1} \symm_{\grp} f := \{ f' \in \fclass : \symm_{\grp}f' = \symm_{\grp} f \}$. In essence, the constraints \eqref{eq:density:matching} are integrals over one or more equivalence classes. $\invf{q}\circ\symm_{\grp}$ is constant on any equivalence class, while $q$ may vary arbitrarily subject to \eqref{eq:density:matching}. 
Inspection of \eqref{eq:symm:gap} indicates that the symmetrization gap is zero if and only if $q$ is constant on each $\grp$-induced equivalence class of $F$. 
Conversely, the more $q$ varies across each equivalence class, the higher the gap. 

\textbf{Symmetrization and compression via other means.} The benefits of compression are not limited to symmetrization via averaging. \emph{Any non-injective, \ginv\ map $\psi$ will have a non-zero $\psi$-gap}. For example, each of $\sup_{g \in \grp} f(gX)$, $\inf_{g \in \grp} f(gX)$, and $\max\{0,\invf{f}(X)\}$ satisfies the criteria. 

\subsection{Practical Feature Averaging}

In practice, the expectation computed in FA may be computationally intractable. Instead, one may sample a set of $k$ transformations with which to average the function output. While this will not output the exact expectation, it still takes advantage of a simplification of the function space via \cref{lem:pushforward:KL}, by aggregating functions that have some probability of being mapped to the same approximately averaged function. To formalize the idea, let $g^k = \{g_1,g_2,\dotsc,g_k\}$ be a set of elements of $\grp$, and $G^k$ a random realization sampled i.i.d.\ from $\haar$. Let $\symm_{g^k} f(x) = k^{-1} \sum_{j\leq k}f(g_j x)$ denote the approximate symmetrization of $f$ by $g^k$. Finally, let $\invfMC{Q}_{g^k} = Q \circ \symm_{g^k}^{-1}$ denote the image of a distribution $Q$ on $\fclass$ under $\symm_{g^k}$. The following result is a consequence of the fact that \cref{lem:pushforward:KL} is true for every $g^k$, and that for $g^{k+1}=g^k\cup\{g_{k+1}\}$, $\symm_{g^{k+1}}f(x) = f(g_{k+1}x) + \frac{k}{k+1}\symm_{g^k}f(x)$.

\begin{proposition} \label{prop:kl:chain}
  Assume the conditions of \cref{lemma:KL:gen}. Let $G_{s}=G_1,G_2,\dotsc$ be a sequence of elements sampled i.i.d.\ from $\haar$. Then with probability one over $G_s$, 
  \begin{align*}
  	 \KL{Q}{P} & \geq \KL{\invfMC{Q}_{G^1}}{\invfMC{P}_{G^1}} \geq \dotsb \\ 
  	   & \geq \KL{\invfMC{Q}_{G^k}}{\invfMC{P}_{G^k}} \geq \dotsb \\
  	   & \geq \KL{\invf{Q}}{\invf{P}} \;.
  \end{align*}
\end{proposition}
As with practical DA, the interplay between SGD and approximate symmetrization remains an open question. However, \cref{prop:kl:chain} makes it clear that at test time, FA---even approximate---is favored. 

\textbf{Computing the symmetrized KL.} One drawback to the generic applicability of FA is the difficulty of computing $\KL{\invf{Q}}{\invf{P}}$ within current approaches to specifying $Q$ and $P$ on neural networks. Specifically, in the approach pioneered by \citet{langford2002pacbayes} and refined by \citet{dziugaite2017nonvacuous} is (roughly) as follows: $P$ is a mean zero uncorrelated multivariate Gaussian distribution on the weights of the network; $Q$ is an uncorrelated multivariate Gaussian distribution, with mean equal to the trained weights and variances optimized to minimize the PAC-Bayes bound. Given that the the network represents a non-linear function, $\KL{\invf{Q}}{\invf{P}}$ cannot be computed in closed form. Whether there is a feasible alternative method to specifying $P$ and $Q$ that would allow for computation of $\KL{\invf{Q}}{\invf{P}}$ remains an open question. 
We give an example of when it can be computed with a linear model in \cref{sec:examples} with a linear model. 

Despite this drawback, the symmetrization gap in the theoretical bounds appears to have real effects on generalization, as shown by the experiments in \cref{sec:experiments}. 

\subsection{PAC-Bayes Bounds}

As discussed in \cref{sec:da:ood}, DA symmetrizes the loss function, which does not guarantee that the learned function $f^*$ will be $\grp$-invariant. Moreover, the generalization error of the learned predictor $f^*$ will be estimated on untransformed test data, precluding randomized prediction distributions $Q$ based on $f^*$ from concentrating on $\invf{\fclass}$. That is, the PAC-Bayes bound for DA does not benefit directly from the symmetrization gap. 

Conversely, FA takes advantage of the symmetrization gap. When the empirical risk $\eRisk(Q,\trdata)$ is close to zero, which will be the case for a trained neural network, the symmetrization gap is the primary contributor to reductions in the PAC-Bayes generalization error bound. When the bound is nonvacuous, \emph{the symmetrization gap is a measurement of the benefit of invariance.}

We formalize these statements in an ordering of the PAC-Bayes generalization upper bounds. Let $B_0$ be the upper bound on the right-hand side of \eqref{eq:catoni:bound}, with $B_{\text{\rm DA}}$ and $B_{\text{\rm FA}}$ corresponding to the upper bounds for DA (using the augmented empirical risk $\eRiskAug(Q,\trdata)$) and FA (using $\KL{\invf{Q}}{\invf{P}}$), respectively. Finally, let $B_{\text{\rm DA \#}}$ denote the computationally intractable bound for DA given in \cref{appx:tighter:pacbayes:da}. 

\begin{theorem} \label{thm:bound:order}
  Assume the conditions of \cref{thm:catoni:bound}, and also that $\dgd$ is \ginv. Then $B_{\text{\rm FA}} \leq B_{\text{\rm DA \#}} \leq B_{\text{\rm DA}} = B_{0}$.
\end{theorem}

Of course, without corresponding lower bounds, this does not imply a strict ordering on generalization error. However, the upper bounds are informative, they should carry some information about relative performance. We demonstrate this empirically in \cref{sec:experiments}. 

\section{Example: Permutation-Invariant Linear Regression} \label{sec:examples}

The following example is a simple toy model, but it adheres to what may be done in practice. Specifically, consider linear regression $f_w(X) = w^{\top} X$, $w \in \bbR^k$, with the PAC-Bayes procedure of \citet{dziugaite2017nonvacuous}: estimate $\hat{w}$ to optimize some loss function; define $Q$ as a $k$-dimensional normal distribution with mean $\hat{w}$, covariance $S = s^2 I_k$, and $P$ likewise with mean $\mu$, covariance $\Sigma = \sigma^2 I_k$. Then
\begin{align*}
  \KL{Q}{P} =  \frac{k}{2} \bigg(\frac{s^2}{\sigma^2}-1 + \ln\frac{\sigma^2}{s^2}\bigg) + \frac{||\mu - w||^2_2}{2\sigma^2} \;.
\end{align*}
Alternatively, consider the same model averaged over all permutations of the inputs. Then for any $w$ in the original model, there is the constant vector $\invf{w}1_k = \frac{1}{d!}\sum_{\pi\in\mathcal{S}_d} \pi w = k^{-1} 1_k 1_k^{\top} w$. The image of the prior therefore is equivalent to a 1-dimensional normal distribution with mean $\invf{\mu} = k^{-1} 1_k 1_k^{\top} \mu$ and variance $k^{-1}\sigma^2$, and similarly for the image of the posterior. Therefore,
\begin{align*}
  \KL{\invf{Q}}{\invf{P}} = \frac{1}{2} \bigg(\frac{s^2}{\sigma^2}-1 + \ln\frac{\sigma^2}{s^2}\bigg) + \frac{k(\invf{\mu} - \invf{w})^2}{2\sigma^2} \;.
\end{align*}
In practice, the KL (and various measures of model complexity) is dominated by the terms involving $||w||^2_2$. Focusing on the difference in those terms, by the Cauchy--Schwarz inequality the symmetrization gap is  
\begin{align*}
  \invf{\Delta}(Q\ ||\ P) \approx \frac{1}{2\sigma^2} \sum_{j=1}^k \big((\mu_j - w_j)^2 - (\invf{\mu} - \invf{w})^2 \big) \geq 0 \;.
\end{align*}

We give a further example based on Boolean functions in \cref{appx:example:boolean}.

\section{Experiments} \label{sec:experiments}
\begin{figure*}[bt]
    \begin{center}
    \includegraphics[width=0.3\textwidth]{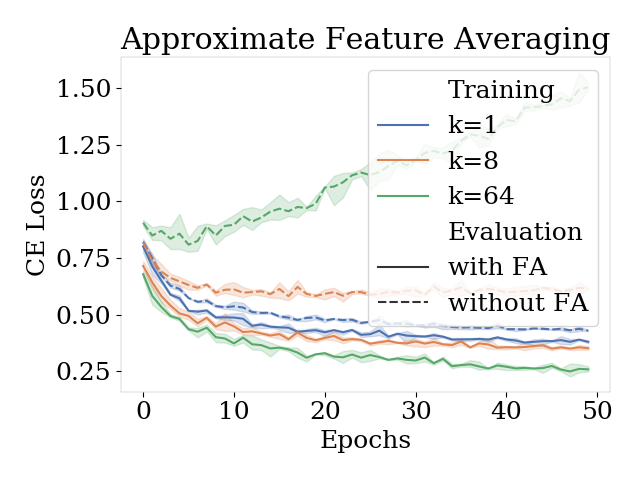}
    \includegraphics[width=0.3\textwidth]{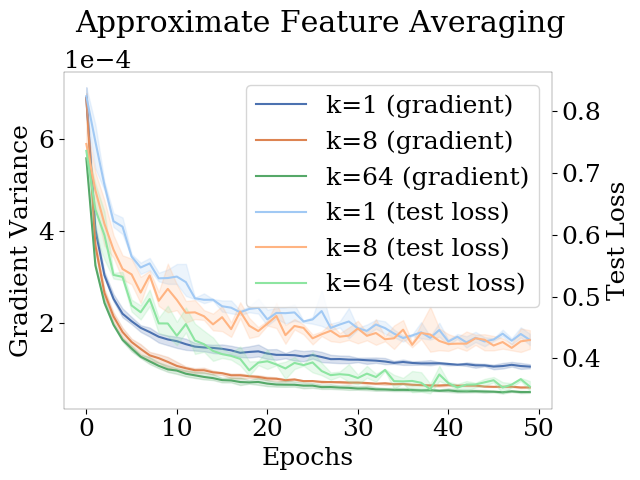}
    \includegraphics[width=0.3\textwidth]{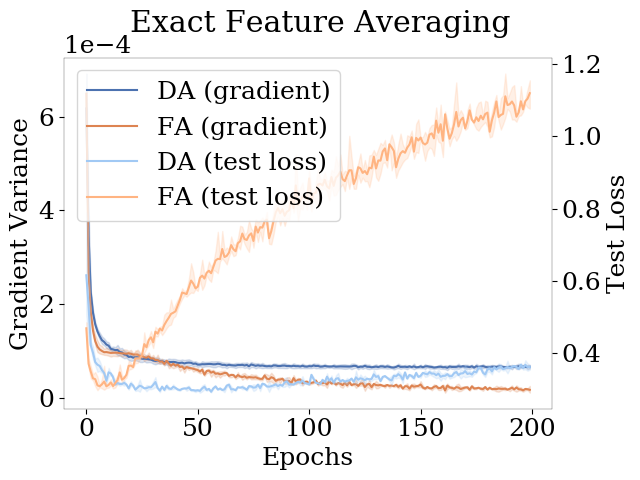}
    \vspace{-\baselineskip}
    \caption{Measurements over the course of training a convolutional neural network using different data augmentation and feature averaging approaches. Left: models are trained with approximate feature averaging using $k$ sampled rotations in the range $\{1, \dots, 360\}$, and then evaluated with and without that averaging scheme.  Middle: per-epoch gradient variance and test loss in the same setting. Right: same dataset and architecture as before, but now augmentation set is composted or rotation by 90 degrees, so feature averaging is exact..}
    \label{fig:acc}
    \end{center}
    \vspace{-0.2in}
\end{figure*}
We provide two examples to illustrate the theoretical results from sections \cref{sec:data:augmentation,sec:feature:averaging}.

\begin{table}[b]
	\vspace{-0.3in}
	\caption{Generalization performance for a permutation-invariant point cloud classification task (see text for details).}
    \label{table:decomp}
    \vspace{-\baselineskip}
    \begin{center}
    \resizebox{0.5\textwidth}{!}{
    \begin{tabular}{ccccc}
    \toprule
         \textbf{Network} & \textbf{Train} & \textbf{Test} &  \textbf{KL} &  \textbf{PAC-Bayes} \\
          & \textbf{Error} & \textbf{Error} & \textbf{Divergence} & \textbf{Bound} \\
         \midrule
         Fully connected & 0.002 & 0.65 & 24957 & 1.75 \\
         Partial-Pointnet & 0.172 & 0.248 & 1992 & 0.67\\
         Pointnet &  0.24 & 0.245 & 944 & 0.533 \\
         \bottomrule
    \end{tabular}
    }
    \end{center}
\end{table}

\subsection{Training Behavior of DA and FA}

In Section 3, we showed that feature averaging reduces variance in both function outputs and gradient steps when compared to data augmentation. We provide a demonstration of how this reduction in variance may play out in practice to ground the previous theoretical analysis and to give the reader a sense of the complexity of analyzing the interplay between feature averaging on gradient descent dynamics. For our evaluation, we train a series of convolutional neural networks on an augmentation of the FashionMNIST dataset. The class of an article of clothing is invariant to rotations: put simply, there is no way of rotating a shoe such that it can be mistaken for a t-shirt. We therefore consider two different augmentations of the dataset by rotations to construct invariant training distributions. 

In the first, we augment the dataset by the 4-element group $\grp$ of 90 degree rotations so that the data-generating distribution $\mathcal{P}$ is invariant to the action of $\grp$, and train a simple convolutional neural network (CNN) once with feature averaging, and once without feature averaging. In this setting, the average over $\grp$ can be computed exactly. Our findings agree with the results of \cref{sec:feature:averaging}: exact FA leads to a reduction in gradient variance, and also to lower training loss. However, the model trained with FA demonstrates overfitting, suggesting that the reduction in variance obtained by exact FA may not always be desirable during training. 

We next consider an additional augmentation of FashionMNIST via the group $\grp$ of rotations in the set $\{1^\circ, \dots, 360^\circ\}$. In this setting, we perform approximate feature averaging with $k$ samples, where $k << |\grp|$. We observe that the model trained with FA becomes increasingly dependent on feature averaging to obtain a low loss: the loss of each individual function computed by the network increases during training, and it is only when averaging over orbits that the network attains the lowest loss. In other words, the trajectory of the models trained with approximate feature averaging converge to regions of parameter space that don't correspond to functions that attain low loss when evaluated without feature averaging, and so may be quite different from the parameters learned by data augmentation. 

\subsection{Generalization in Neural Networks}

We next provide a demonstration of the effect of invariance on PAC-Bayesian bounds for neural networks. We use the ModelNet10 dataset, which consists of LiDAR point cloud data for 10 classes of household objects. This dataset exhibits permutation invariance: the LiDAR reading is stored as a sequence of points defined by $\{x,y,z\}$ coordinates, and the order in which the points are listed is irrelevant to the class.  We consider three different architectures: a PointNet-like architecture \citep{qi2017pointnet}, which is invariant to permutations; a partitioned version of the PointNet architecture which is invariant to subgroups of the permutation group (details in the Appendix); and a fully connected model where the invariant pooling operation in the PointNet is replaced by a fully-connected layer. The invariance in the network is implemented via a max-pooling layer instead of an averaging layer and so is not a direct application of feature averaging; however, the results of \cref{eq:kl:inequality:pushforward} would apply, were we able to compute the PAC-Bayesian bound for the model exactly.

We compute the PAC-Bayes bounds following the procedure in \citet{dziugaite2017nonvacuous}: we convert a deterministic network to a stochastic network by adding Gaussian noise to the weights, and then train this stochastic model using a differentiable surrogate loss that bounds the true PAC-Bayes bound. After this training procedure converges, we then compute the true PAC-Bayes bound. We attain an ordering consistent with the observations presented in the previous section: the invariant architecture attains the lowest bound, followed by the partially invariant architecture, and finally followed by the fully connected network. We provide a decomposition of the distinct terms in the bound in \cref{table:decomp}.

\section{Practical Implications and Conclusions} \label{sec:conclusion}

We refer back to \cref{tab:theory:summary} for a summary of our theoretical results. A few practical guidelines emerge. 

\textbf{Train with approximate data augmentation or feature averaging.} The reduction in variance of risk estimates and their gradients obtained by averaging over $\grp$ appears to be beneficial to training, though too much variance-reduction seems undesirable. Based on the experiments in \cref{sec:experiments}, we advocate for training with approximate FA or DA. 

\textbf{Use feature averaging at test/deployment time.} With a convex loss function, the generalization error $\risk(f)$ of a feature-averaged model is no worse, and possibly better, than that of its non-averaged counterpart, even when DA was used for training. Even with non-convex loss, a randomized prediction rule $Q$ has looser generalization bounds than its $\grp$-averaged counterpart $\invf{Q}$. Because of this, even a model trained with DA should generalize better when its outputs are averaged over $\grp$ at test time. The experiments in \cref{sec:experiments} demonstrate this empirically.

\section*{Acknowledgements} 
MK has received funding from the European Research Council (ERC)
under the European Union’s Horizon 2020 research and innovation programme
(grant agreement No.~834115).

\bibliography{references}
\bibliographystyle{icml2020}

\clearpage

\appendix
\onecolumn


\section{Proofs} \label{appx:proofs}

\begin{proof}[Proof of \cref{theorem:symmgd}]

Let $w \in V^*$, and suppose that $w$ is not invariant under the action of $\grp$. Let $\invf{w} = \bbE_{G\sim\haar}[\rho^*_G w]$, which is \ginv\ by construction. Because $\calX$ spans $V$, $w - \invf{w} \neq 0$ implies that $w \neq \invf{w}$. 

Consider the minimizer
\begin{align*}
  \hat{w} = \argmin_{w \in V^*} \eRiskAug(f_w,\trdata) = \argmin_{w \in V^*} \frac{1}{n}\sum_{i=1}^n \bbE_{G\sim\haar}[\loss(w^{\top}\rho_G X_i,Y_i)] \;,
\end{align*}
which is unique because $\loss$ is strictly convex by assumption. Assume that $\hat{w}$ is not \ginv. 
Applying Jensen's inequality, we have
\begin{align*}
  \eRiskAug(f_{\hat{w}},\trdata) & = \frac{1}{n}\sum_{i=1}^n \bbE_{G\sim\haar}[\loss(\hat{w}^{\top} \rho_G X_i,Y_i)] \\
  & > \frac{1}{n}\sum_{i=1}^n \loss(\bbE_{G\sim\haar}[\hat{w}^{\top}\rho_G X_i,Y_i)] \\
  & = \frac{1}{n}\sum_{i=1}^n \loss(\bbE_{G\sim\haar}[(\rho_{G^{-1}}^* \hat{w})]^{\top} X_i,Y_i)] \\
  & = \frac{1}{n}\sum_{i=1}^n \loss(\invf{\hat{w}} X_i, Y_i) = \eRiskAug(f_{\invf{\hat{w}}},\trdata) \;,
\end{align*}
which cannot be the case because $\hat{w}$ minimizes $\eRiskAug$. Therefore, $\hat{w}$ must be \ginv.
\end{proof}

\subsection{Proof of \texorpdfstring{\cref{thm:pac:bayes:da}}{Theorem 4}}

The proof of our PAC-Bayes bound for data augmentation makes use of the following result due to \citet{leveretal2013tighterPACbayes}.

\begin{theorem}[\citet{leveretal2013tighterPACbayes}, Theorem 1] \label{lem:lever:bound}
  For any functions $A(f)$, $B(f)$ over $\fclass$, either of which may be a statistic of the training data $\trdata$, any distribution $P$ over $\fclass$, any $\delta \in (0,1]$, any $t > 0$, and a convex function $\scD : \bbR \times \bbR \to \bbR$, with probability $\dgd^n$ at least $1 - \delta$, for all distributions $Q$ on $\fclass$,
  \begin{align} \label{eq:lever:bound}
    \scD\big( \bbE_{f\sim Q}[A(f)],\bbE_{f\sim Q}[B(f)] \big) \leq \frac{1}{t} \bigg( \KL{Q}{P} + \log \frac{\calL_P}{\delta}  \bigg) \;,
  \end{align}
  where $\calL_P: = \bbE_{\trdata\sim \dgd, f\sim P}[e^{t\scD(A(f),B(f))}]$ is the Laplace transform of $\scD(A(f),B(f))$.
\end{theorem}

As \citet{leveretal2013tighterPACbayes} discuss, many PAC-Bayes bounds in the literature can be obtained as special cases of \cref{lem:lever:bound}, including Catoni's bound in \cref{thm:catoni:bound}. In that case, which applies to 0-1 loss, $t=n$, $A(f) = \eRisk(f,\trdata)$, $B(f) = \risk(f)$, and 
\begin{align}
  \scD_C(q,p) & := -\log (1-p(1-e^{-C})) - C q \;, \quad q,p \in (0,1), \ C > 0 \\ 
    & = -\log \bbE_{Z \sim \text{Bern}(p)}[e^{-CZ}] - Cq \;.
\end{align}
Basic calculations show that with these quantities, $\calL_P=1$.

Recall that
\begin{align} 
  \eRisk(f,\trdata) &:= \frac{1}{n} \sum_{i=1}^n \loss(f(X_i),Y_i) \label{eq:risk:1} \\
  \eRiskAug(f,\trdata) &:= \frac{1}{n} \sum_{i=1}^n \bbE_{G\sim\haar}[\loss(f(G X_i),Y_i)] \label{eq:risk:2} \\
  \eRiskAugMC(f,\trdata) &:= \frac{1}{nm} \sum_{i=1}^n \sum_{j=1}^m \loss(f(G_{ij} X_i),Y_i) \label{eq:risk:3} \;.
\end{align}

Let $(G_{ij})$ denote the collection of $m\cdot n$ random augmentation transformations sampled i.i.d.\ from $\haar$. 
\begin{lemma} \label{lem:div:bounds}
  Let $\loss$ be binary loss, $P$ any distribution on $\fclass$, and assume that $\dgd$ is \ginv. Then
  \begin{align} \label{eq:div:bound:aug}
    \bbE_{f\sim P}\big[\bbE_{\trdata\sim\dgd}\big[ e^{n\scD_C(\eRiskAug(f,\trdata),\risk(f))}   \big] \big] & \leq \bbE_{f\sim P}\big[\bbE_{\trdata\sim\dgd}[e^{n\scD_C(\eRisk(f,\trdata),\risk(f))}] \big] = 1
  \end{align}
  and
  \begin{align} \label{eq:div:bound:aug:mc}
    \bbE_{f\sim P}\big[\bbE_{\trdata\sim\dgd}\big[ e^{n\scD_C(\eRiskAugMC(f,\trdata),\risk(f))}   \big] \big] & \leq \bbE_{f\sim P}\big[\bbE_{\trdata\sim\dgd}[e^{n\scD_C(\eRisk(f,\trdata),\risk(f))}] \big] = 1 \;.
  \end{align}
\end{lemma}
\begin{proof}
  Since the observations $(X_i,Y_i)$ are i.i.d., the expectation over $\trdata$ on the left-hand side of \eqref{eq:div:bound:aug} requires evaluating $\bbE_{\trdata\sim\dgd}\big[ e^{-C\bbE_{G\sim\haar}[ \loss(f(X_i),Y_i))]} \big]$. Using the convexity of $e^{-x}$, Jensen's inequality and Fubini's theorem yield
  \begin{align} \label{eq:augrisk:laplace}
    \bbE_{(X_i,Y_i)\sim\dgd}\big[ e^{-C\bbE_{G\sim\haar}[ \loss(f(G X_i),Y_i))]} \big] 
      & \leq \bbE_{(X_i,Y_i)\sim\dgd}\big[ \bbE_{G\sim\haar} \big[ e^{-C \loss(f(G X_i),Y_i))} \big] \big] \\
      & = \bbE_{G\sim\haar} \big[ \bbE_{(X_i,Y_i)\sim\dgd}\big[  e^{-C \loss(f(G X_i),Y_i))} \big] \big] \;. \nonumber
  \end{align}
  Now, $\grp$-invariance of $\dgd$ implies that $\bbE_{(X_i,Y_i)\sim\dgd}[h(gX_i,Y_i)] = \bbE_{(X_i,Y_i)\sim\dgd}[h(X_i,Y_i)]$ for all measurable functions $h : \calX \times \calY \to \bbR_+$ and all $g\in\grp$, which extends to {independent} random $G$ by Fubini's theorem. Therefore,
  \begin{align*}
    \bbE_{G\sim\haar} \big[ \bbE_{(X_i,Y_i)\sim\dgd}\big[  e^{-C \loss(f(G X_i),Y_i))} \big] \big]
      = \bbE_{(X_i,Y_i)\sim\dgd}\big[  e^{-C \loss(f(X_i),Y_i))} \big] = \bbE_{Z\sim\text{Bern}(\risk(f))}[e^{-CZ}] \;,
  \end{align*}
  which implies \eqref{eq:div:bound:aug}.

  For the second inequality \eqref{eq:div:bound:aug:mc}, observe that by Jensen's inequality,
  \begin{align*}
    \bbE_{\trdata\sim\dgd}\big[ e^{-nC\eRiskAugMC(f,\trdata)} \big]
      & = \prod_{i=1}^n \bbE_{(X_i,Y_i)\sim\dgd}\bigg[ \bbE_{(G_{ij})_{j=1}^m\sim\haar}\bigg[ \exp\bigg(-\frac{C}{m}\sum_{j=1}^m \loss(f(G_{ij}X_i),Y_i) \bigg) \bigg] \bigg] \\
      & \leq \prod_{i=1}^n \bbE_{(X_i,Y_i)\sim\dgd}\bigg[ \bbE_{(G_{ij})_{j=1}^m\sim\haar}\bigg[ \frac{1}{m} \sum_{j=1}^m e^{-C \loss(f(G_{ij}X_i),Y_i) } \bigg] \bigg] \\
      & = \prod_{i=1}^n \bbE_{(X_i,Y_i)\sim\dgd}\big[ \bbE_{G\sim\haar} \big[ e^{-C \loss(f(G X_i),Y_i) } \big] \big]
  \end{align*}
  Using the $\grp$-invariance of $\dgd$ once again, we have
  \begin{align*}
    \bbE_{\trdata\sim\dgd}\big[ e^{-nC\eRiskAugMC(f,\trdata)} \big] \leq \bbE_{\trdata\sim\dgd}\big[ e^{-nC\eRisk(f,\trdata)} \big] = \big(\bbE_{Z\sim\text{Bern}(\risk(f))}[e^{-CZ}] \big)^n \;,
  \end{align*}
  which implies \eqref{eq:div:bound:aug:mc}.
\end{proof}

\begin{proof}[Proof of \cref{thm:pac:bayes:da}]
  \Cref{thm:pac:bayes:da} follows from \cref{lem:lever:bound,lem:div:bounds}. In particular, observe that the expectation of any of the risks \eqref{eq:risk:1}--\eqref{eq:risk:3} over $\trdata$ and $f\sim Q$ is $\risk(Q)$. Therefore, using any of those risks as $A(f)$ in \cref{lem:lever:bound} with $B(f) = \risk(f)$ will result in valid a PAC-Bayes bound; the only quantity that changes between the three situations is $\calL_P$ in \eqref{eq:lever:bound}. \Cref{lem:div:bounds} establishes that $\calL_P$ when $A(f)$ is either of $\eRiskAug(f,\trdata)$ or $\eRiskAugMC(f,\trdata)$ is upper-bounded by $\calL_P$ when $A(f) = \eRisk(f,\trdata)$, which is equal to 1.

  The particular bound \eqref{eq:catoni:bound} follows from algebraic manipulations of \eqref{eq:lever:bound}.
\end{proof}

\subsection{Proof of \texorpdfstring{\cref{prop:empirical:risk:order}}{Proposition 5}} \label{appx:proof:prop:empirical:risk:order}

\begin{proof}[Proof of \cref{prop:empirical:risk:order}]
  Let $\grp$ be a group with some probability measure $\haar$, and $\fclass$ a class of functions $f : \calX \to \bbR$. Let $\loss : \bbR \times \bbR \to \bbR_+$ be a loss function such that $\loss(f(\argdot),\argdot)\in L_2(\dgd)$ for every $f \in \fclass$. Then the augmented risk of any function $f \in \fclass$ is
  \begin{align*}
    \eRiskAug(f,\trdata) = \frac{1}{n}\sum_{i=1}^n  \bbE_{G\sim\haar}[\loss(f(G  X_i),Y_i)] \;.
  \end{align*}
  If $\loss$ is convex in the first argument, then by Jensen's inequality,
  \begin{align} \label{eq:jensens}
    \bbE_{G\sim\haar}[\loss(f(G X_i),Y_i)] \geq \loss(\bbE_{G\sim\haar}[f(G X_i)],Y_i) \;, \quad i = 1,2,\dotsc,n \;.
  \end{align}
  On the other hand, the $\grp$-symmetrization of $f(X)$ is $\invf{f}(X) = \bbE_{G\sim\haar}[f(G X)]$, with augmented risk
  \begin{align*}
    \eRiskAug(\invf{f},\trdata) & = 
      \frac{1}{n}\sum_{i=1}^n  \bbE_{G\sim\haar}[\loss(\bbE_{G\sim\haar}[f(G X_i)],Y_i)] \\
      & =\frac{1}{n}\sum_{i=1}^n  \loss(\bbE_{G\sim\haar}[f(G X_i)],Y_i) \\
      & = \eRisk(\invf{f},\trdata) \;.
  \end{align*}
  Combined with \eqref{eq:jensens}, the reduction in empirical augmented risk 
  follows. The reduction in $\eRiskAug(Q,\trdata)$ follows trivially. 

  The variance-reduction 
  is established by extending the argument in the proof of \cref{prop:da:variance:reduction}. Specifically, by the conditional Jensen's inequality,
  \begin{align*}
    \Var_{\trdata\sim\dgd^n}\big[  \eRiskAug(f,\trdata)  \big] = \Var[\bbE[\eRiskAug(f,\trdata) \mid \Orbit^n]] \geq \Var[ \bbE[ \eRisk(\invf{f},\trdata) \mid \Orbit^n]  ] = \Var_{\trdata\sim\dgd^n}\big[ \eRisk(\invf{f},\trdata) \big] \;.
  \end{align*}
\end{proof}

\subsection{Proof of \texorpdfstring{\cref{lem:pushforward:KL,lemma:KL:gen}}{Lemma 6 and Theorem 7}} \label{appx:proof:lem:pushforward:KL}

\def\tmu{\tilde{\mu}}
\def\tnu{\tilde{\nu}}

The proof of \cref{lem:pushforward:KL} relies on the chain rule of relative entropy. 
Let two probability measures, $\tmu \ll \tnu$ defined on the product space $(E_1 \times E_2, \calE_1 \otimes \calE_2)$, have marginal measures $\tmu_1\ll \tnu_1$ on $(E_1,\calE_1)$ (respectively, $\tmu_2 \ll \tnu_2$ on $(E_2,\calE_2)$) and regular conditional probability measures $\tmu_{2|1}\ll\tnu_{2|1}$ (resp.\ $\tmu_{1|2}\ll\tnu_{1|2}$). Recall the chain rule of relative entropy is
\begin{align} \label{eq:chain:rule}
  \KL{\tmu}{\tnu} = \KL{\tmu_1}{\tnu_1} + \bbE_{\tmu}\bigg[ \log \frac{d\tmu_{2|1}}{d\tnu_{2|1}}  \bigg] = \KL{\tmu_2}{\tnu_2} + \bbE_{\tmu}\bigg[ \log \frac{d\tmu_{1|2}}{d\tnu_{1|2}}  \bigg] \;.
\end{align}
Observe that each of the terms in the equalities is non-negative.

\begin{proof}[Proof of \cref{lem:pushforward:KL}]

  Given probability measures on $(E_1,\calE_1)$ $\mu \ll \nu$ (with density $m$ such that $\mu = m\cdot \nu$) and a measurable map $\psi : (E_1,\calE_1) \to (E_2,\calE_2)$, construct the probability measure $\tmu$ on $(E_1 \times E_2, \calE_1 \otimes \calE_2)$ as
  \begin{align*}
    \tmu(A \times B) = \mu(A \cap \psi^{-1}B) = \int_A \mu(dx_1) \int_{B} \delta_{\psi(x_1)}(dx_2) \;, \quad A \in \calE_1,\ B \in \calE_2 \;,
  \end{align*}
  and likewise for $\tnu$. Then in the notation of \eqref{eq:chain:rule}, $\tmu_1 = \mu\ll \nu = \tnu_1$, and $\tmu_{2|1} = \delta_{\psi(x_1)} = \tnu_{2|1}$. Therefore,
  \begin{align}
    \KL{\tmu}{\tnu} = \KL{\tmu_1}{\tnu_1} = \KL{\mu}{\nu} \;.
  \end{align}
  Alternatively, $\tmu_2 = \mu\circ\psi^{-1}$, $\tnu_2 = \nu\circ\psi^{-1}$, and it is straightforward to show that
  \begin{align}
    \bbE_{\tmu}\bigg[ \log \frac{d\tmu_{1|2}}{d\tnu_{1|2}}  \bigg] = \bbE_{\tmu}\bigg[ \log \frac{d\tmu_{1}}{d\tnu_{1}}  \bigg] - \bbE_{\tmu}\bigg[ \log \frac{d\tmu_{2}}{d\tnu_{2}}  \bigg] = \bbE_{\mu}\bigg[ \log \frac{m}{m\circ\psi}  \bigg] = \Delta_{\psi}(\mu\ ||\ \nu) \geq 0. \;.
  \end{align}
  Therefore,
  \begin{align}
    \KL{\tmu}{\tnu} = \KL{\mu}{\nu} = \KL{\mu\circ\psi^{-1}}{\nu\circ\psi^{-1}} + \Delta_{\psi}(\mu\ ||\ \nu) \;.
  \end{align}

\end{proof}

\begin{proof}[Proof of \cref{lemma:KL:gen}]
  For $\calX$ a compact metric space and $\calY$ a Polish space, the space $\fclass = C(\calX,\calY)$ of continuous functions $f : \calX \to \calY$ is a Polish space, and therefore it (along with its Borel $\sigma$-algebra $\borel(C(\calX,\calY))$) is a standard Borel space. 
  For a group $\grp$ acting measurably on $\calX$, the symmetrization operator $\symm_{\grp} : \fclass \to \invf{\fclass}$ is measurable, and the product space $(\fclass \times \invf{\fclass}, \borel(\fclass)\otimes\borel(\invf{\fclass}))$ is a standard Borel space. Thus, the conditions of \cref{lem:pushforward:KL} are satisfied and the result follows.
\end{proof}


\section{Examples, Counterexamples, Tighter Bounds}

\subsection{Permutation-invariant Boolean Function} \label{appx:example:boolean}

As an illustrative example, we consider the task of learning a permutation-invariant Boolean function. We consider the following toy learning algorithm. For a training set $\trdata$, each observation of which is a pair $(X_i,Y_i) \in \{0,1\}^k \times \{0,1\}$, the algorithm outputs a sample from $Q_\trdata$, the uniform distribution over all $k$-ary Boolean functions which agree with $\trdata$. If the full function space under consideration is the set of $k$-ary Boolean functions $\fclass = \{f : \{0,1\}^k \rightarrow \{0,1\} \}$, then $|\fclass| = 2^{2^k}$. Moreover, the number of Boolean functions consistent with a training data set containing $|\trdata|$ unique binary vectors is $2^{2^k - |\trdata|}$. Thus, letting $P$ denote the uniform distribution over $k$-ary Boolean functions,  
\begin{align*}
\KL{Q}{P} 
 = \log_2 \frac{2^{2^k}}{2^{2^k - |\trdata|}} 
= |\trdata| \leq n \;,
\end{align*}
where for convenience we have used $\log_2$ inside the KL.

In contrast, we can consider the same learning algorithm applied to the class of permutation-invariant Boolean functions.\footnote{Note that the class of permutation-invariant Boolean functions is a strict subset of the Boolean functions, because symmetrization via averaging produces a function with image $[0,1]$ (and thus the Boolean functions are not closed under averaging).} 
The permutation-invariant Boolean functions are those that are constant on all input vectors containing the same number of 1-valued entries, and therefore is equivalent to the set of functions $F_{\textrm{inv}} =\{ f: \{0,\dotsc,k\} \to \{0,1\} \}$, with $|F_{\textrm{inv}}| = 2^{k+1}$. The restriction of the uniform prior $P$ to $F_{\textrm{inv}}$ remains uniform, ${P_{\textrm{inv}}}({f_{\textrm{inv}}})=2^{-(k+1)}$. The restriction of $Q$ depends on $|\trdata|_{\textrm{inv}}$, the number of $j\in\{0,\cdots,k\}$ such that at least observation in $\trdata$ has exactly $j$ 1-valued entries: ${Q_{\textrm{inv}}}({f_{\textrm{inv}}})=2^{-(k+1 - |\trdata|_{\textrm{inv}})}$. 
Thus,
\begin{align*}
  \KL{Q_{\textrm{inv}} }{P_{\textrm{inv}} } = |\trdata|_{\textrm{inv}} \leq |\trdata| \;.
\end{align*}
In this simple case, if the observations are consistent with the assumptions, i.e., the output is constant across input vectors with the same number of 1-valued entries, then the invariant model obtains a KL gap of $|\trdata| - |\trdata|_{\textrm{inv}}$.

\subsection{Counterexamples} \label{appx:counterexamples}

\textbf{Feature averaging and non-convex losses.} We consider the binary classification setting with the zero-one loss and some function class $f$ bounded in $[0,1]$ -- that is $\ell(x, y) = \mathbbm{1}[|f(x) - y| > 1]$. Suppose that there exists some invariance $\grp$ in the data such that $y(x) = y(gx)$ for all $x, g$. Then consider a function which, for some small $\epsilon$, outputs $f(x) = \frac{1}{2} + y\epsilon$ on a $1 - 2\epsilon$ fraction of each equivalence class of the inputs, and $1 - y$ on $2\epsilon$ of the inputs in each equivalence class. Then $\mathbb{E}[f(gx)] = (1 - 2\epsilon) (\frac{1}{2} + y\epsilon) + 2\epsilon (1-y)$. When $y=0$, this expectation is $\frac{1}{2} + \epsilon$, and when $y=1$ it is $\frac{1}{2} [1 - \epsilon - 2\epsilon^2] < \frac{1}{2}$, so the feature-averaged model would have risk 1 whereas the original model had risk 0.

\textbf{Non-uniform data-generating distributions.}  When the data-generating distribution is not uniform over the set $\T$, then performing data augmentation with $\T$ will not necessarily lead to a more accurate estimate of the model's empirical risk. For example, consider the task of learning a function $g$ satisfying $g(x) = g(-x)$, bounded in magnitude by some constant $A$. Suppose, however, that positive numbers are much more likely under the data generating distribution, with $p(\mathbb{R}^+) = 1 - \epsilon$ for small $\epsilon$. Then the function $f(x) = \mathbbm{1}[x>0]g(x)$ will satisfy $\mathbb{E}[\|f(X_S) - g(X_s)\|] \neq \mathbb{E}[ \|f(X_{S^\text{aug}}) - g(X_{S^\text{aug}})\|]$. So the augmented risk is no longer an unbiased estimator of the empirical risk. Further, in this particular case its variance is also higher, as it will be equal to $\frac{1}{2} $Var$(g(x))$, in contrast to $\epsilon \text{Var}(g(x))$.

\subsection{Tighter PAC-Bayes Bounds for Data Augmentation} \label{appx:tighter:pacbayes:da}

Although \cref{thm:pac:bayes:da} establishes that the i.i.d.\ PAC-Bayes bound \eqref{eq:catoni:bound} is valid for exact DA, the proof of \cref{thm:pac:bayes:da} indicates that a tighter bound is possible. In particular, recall that when $\dgd$ is \ginv\ \citep{invariantdistributions,chen2019invariance},
\begin{align*}
  \bbE_{G\sim\haar}[\loss(f(GX),Y)] = \bbE_{(X,Y)\sim\dgd}[\loss(f(X),Y) \mid \Orbit] := \invf{\loss}_f(\Orbit) \;.
\end{align*}
$\invf{\loss}_f(\Orbit)$ is a random variable, the average loss on the random orbit with representative $\Orbit$, whose distribution is induced by $\dgd$. Therefore, we can write $\mathcal{L}_P$ in \eqref{eq:lever:bound} as
\begin{align*}
  \mathcal{L}_P = \bbE_{f\sim P}\bigg[ \bigg(\frac{ \bbE_{\Phi\sim\dgd} \big[ e^{-C\invf{\loss}_f(\Orbit)} \big]}{\bbE_{Z\sim\text{Bern}(\risk(f))}[e^{-CZ}]} \bigg)^n \bigg] \leq 1 \;.
\end{align*}
In general, this cannot be computed in closed form. However, it might be possible to estimate using the data (with appropriate modifications to the resulting bound) and samples $f \sim P$.


\section{Computation Details for PAC-Bayes Bounds}
\label{appx:computations}

PAC-Bayes bounds for neural networks are computed via the following procedure: a deterministic neural network is trained to minimize the cross-entropy loss on the dataset. After it has reached a suitable training accuracy, we use these parameters as the initialization for the means and variances of the stochastic neural network weights used for the PAC-Bayes bounds. We directly optimize a surrogate of the PAC-Bayes bound (using the cross-entropy loss instead of the zero-one accuracy and using the reparameterization trick to get the derivatives of the variance parameters). The exact computation of the PAC-Bayes bound uses the union bound and discretization of the PAC-Bayes prior as described in \citep{dziugaite2017nonvacuous}. Reported values are at optimization convergence.

\subsection{Experiment parameters and computation details}

The experiment code is provided with the paper submission, but we describe here at a high level the different models used in our empirical evaluations.

\textbf{FashionMNIST CNN:} the convolutional network used for FashionMNIST consists of two convolutional layers (with batch norm and max pooling) followed by a single fully connected layer. 

\textbf{LiDAR Permutation-Invariant Network:} we use a scaled-down version of the PointNet architecture \citep{qi2017pointnet}. We include two layers of 1D convolutions followed by a max-pooling layer that selects the maximum over input points for each channel. This layer is followed by two fully-connected layers leading into the final output. 

\textbf{Partially-Invariant Network:} we alter the previous architecture slightly so that it is only invariant to \textit{subgroups} of the permutation group on its inputs. Specifically, we partition the input into 8 disjoint subsets, and apply the previous model's permutation-invariant embedding layers to each partition. The result is a feature representation that is invariant to permutations within each partition of the input, but not between partitions. This representation is then fed through the same architecture. We note that we keep the number of convolutional filters per layer constant, which results in a larger feature embedding by a factor of 8 that is fed into the first fully connected layer. As a result, this model has significantly more parameters than the fully permutation-invariant model.

\textbf{Fully Connected Network:} the max-pooling operator of the previous two architectures is omitted. This network has many more parameters than either of the first two models, and is not invariant to any subgroup of the permutation group.


\section{Additional Empirical Evaluations}

In addition to the results shown in \cref{fig:acc}, we include further plots to characterize training the FA as opposed to DA, and provide some insights here.
\begin{enumerate}
    \item Feature averaging at evaluation uniformly improves the loss function compared to sampling a single input.
    \item Feature averaging at evaluation doesn't appear to significantly harm accuracy (a non-convex loss function), but doesn't see the same improvement as for the cross-entropy loss.
    \item Models trained with feature averaging tend to achieve lower training loss, but in the exact feature averaging setting this improved training loss is accompanied by increased overfitting.
    \item Models trained with feature averaging perform worse over time when evaluated with a single sample. This gap increases as the model is trained.
\end{enumerate}
\begin{figure}
    \centering
     \label{fig:approxaveraging}
    \includegraphics[width=0.3\textwidth]{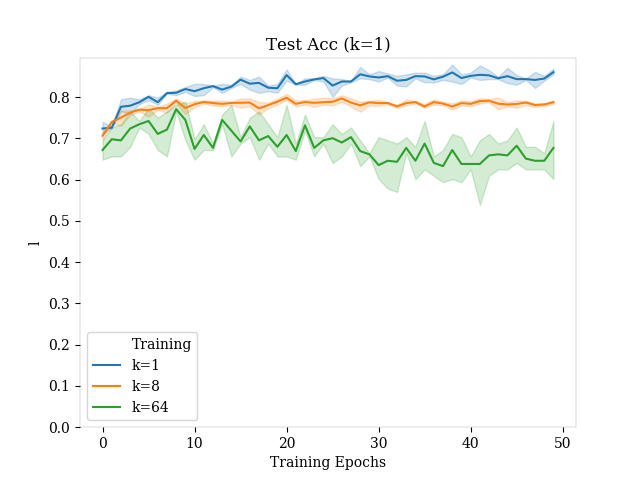}
        \includegraphics[width=0.3\textwidth]{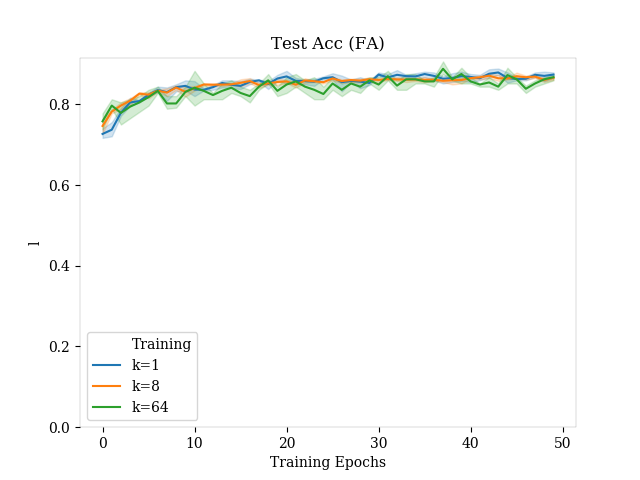}
        \includegraphics[width=0.3\textwidth]{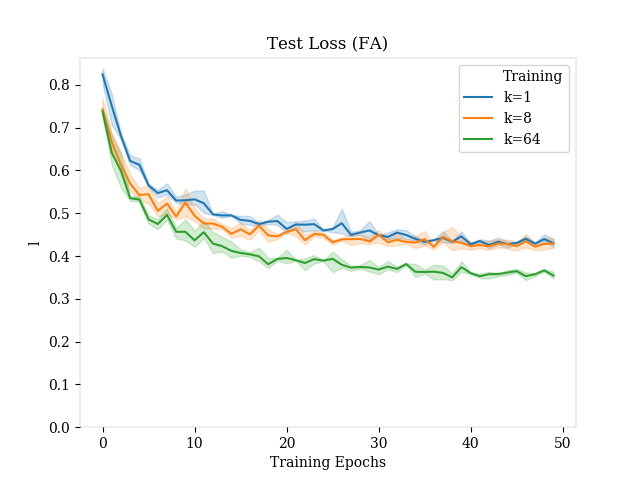}
     \includegraphics[width=0.3\textwidth]{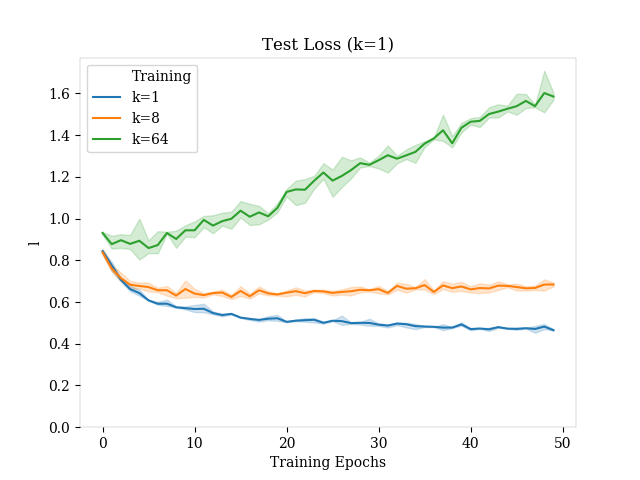}
      \includegraphics[width=0.3\textwidth]{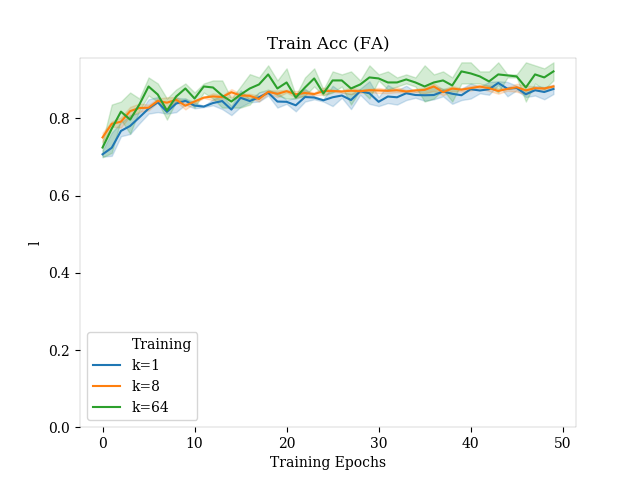}
      \includegraphics[width=0.3\textwidth]{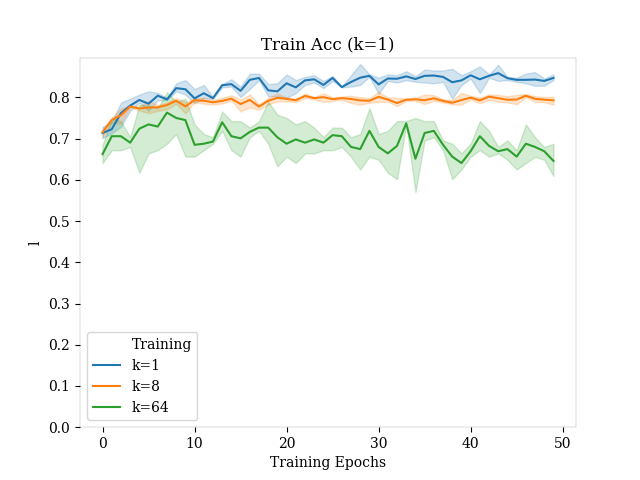}
      
      \includegraphics[width=0.3\textwidth]{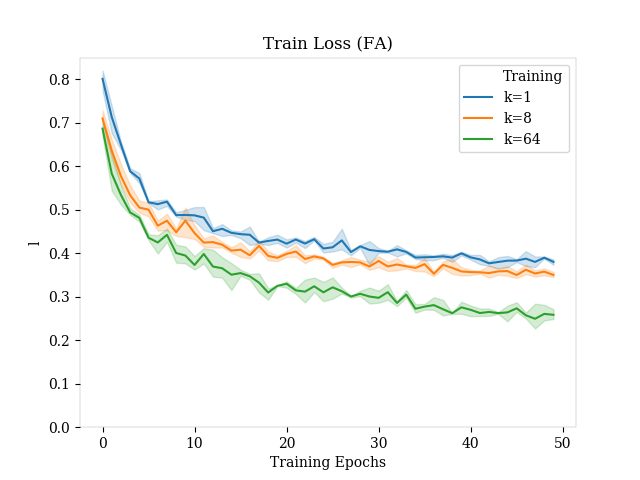}
      \includegraphics[width=0.3\textwidth]{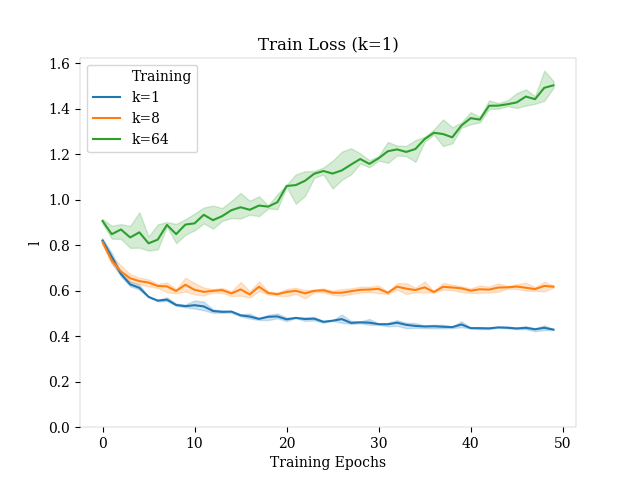}
      \caption{Measures of performance for networks trained with approximate feature averaging. Number of samples used in approximate FA during training range from $k=1$ to $k=64$. FA indicates that the model was evaluated using the same number of samples that it was trained on, while k=1 indicates that a single sample is drawn at evaluation time.}
    \end{figure}
    \begin{figure} \label{fig:exactaveraging}
    \centering
        \includegraphics[width=0.3\textwidth]{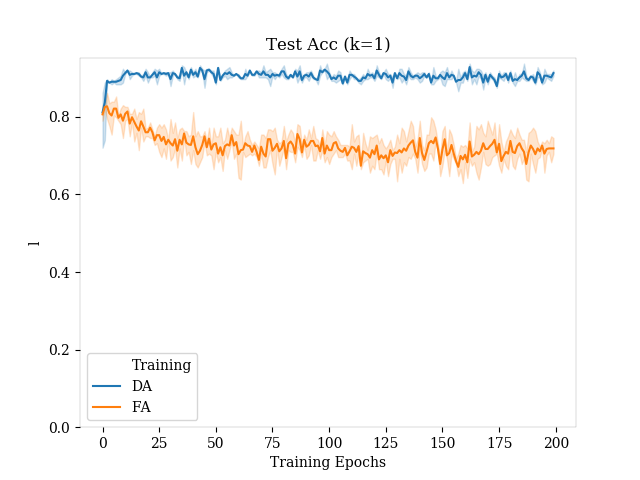}
        \includegraphics[width=0.3\textwidth]{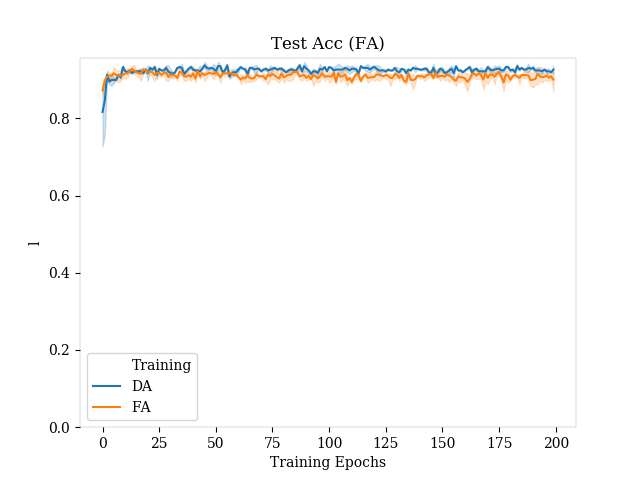}
        \includegraphics[width=0.3\textwidth]{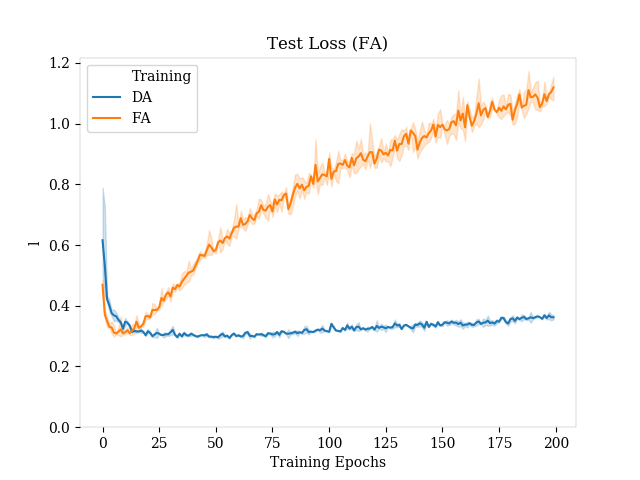}
     \includegraphics[width=0.3\textwidth]{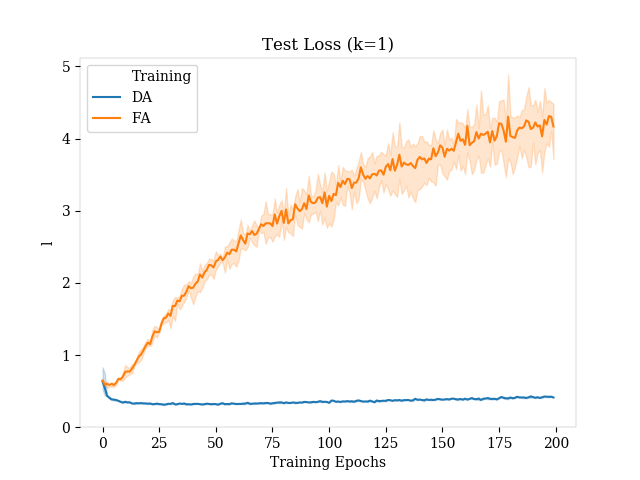}
      \includegraphics[width=0.3\textwidth]{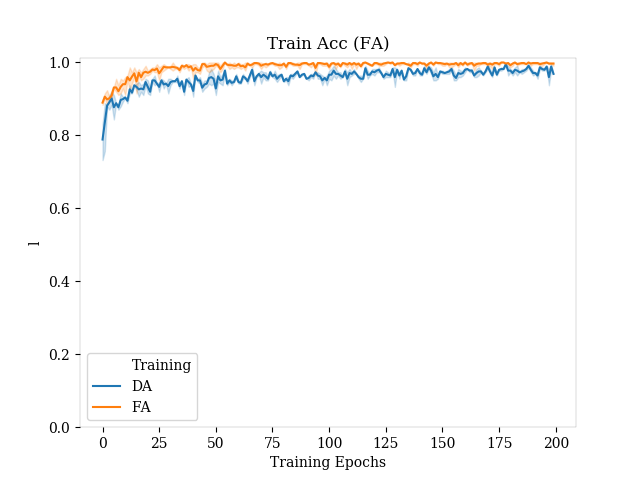}
      \includegraphics[width=0.3\textwidth]{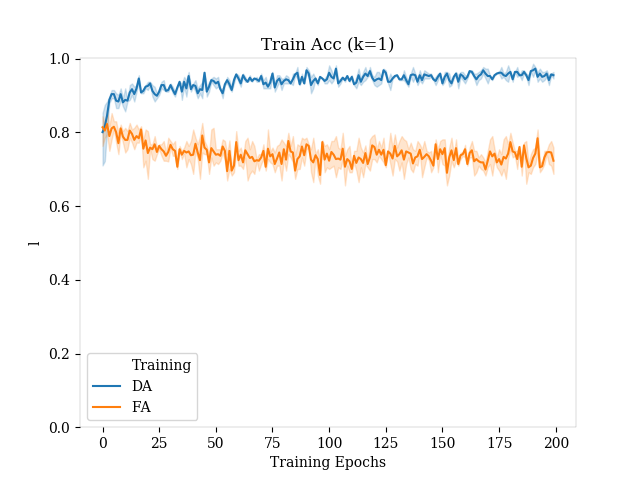}
      
      \includegraphics[width=0.3\textwidth]{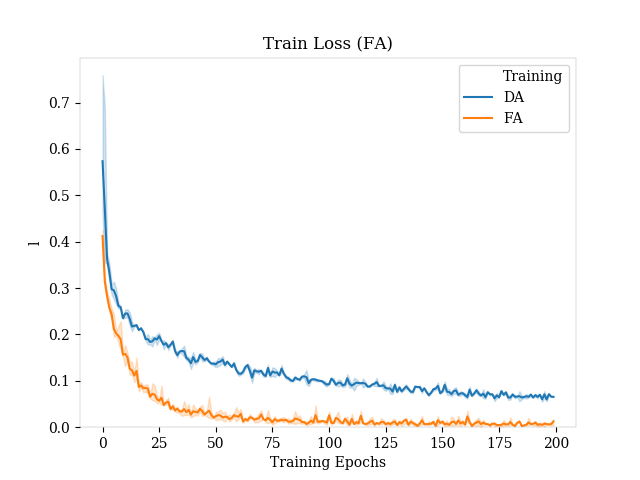}
      \includegraphics[width=0.3\textwidth]{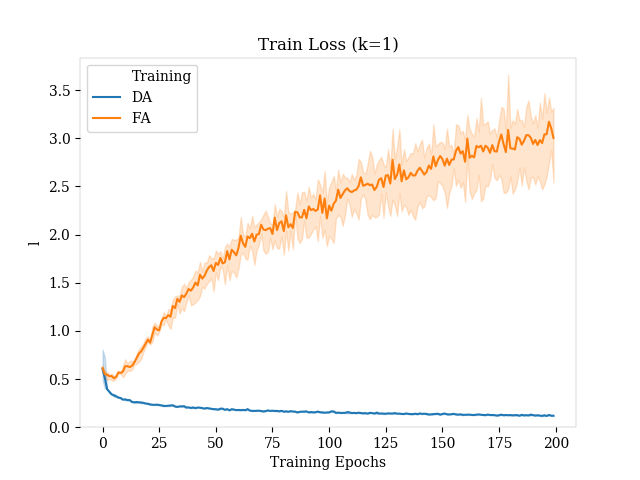}
    \caption{Different measures of performance of networks under different training regimes. Evaluation format (with a single sample or with averaging) is included in title, and training method (trained with data augmentation or feature averaging) is distinguished within each plot by colour.}
   
\end{figure}

\end{document}